\documentclass[times, preprint, 10pt]{elsarticle}

\usepackage{todonotes}
\usepackage{booktabs}
\usepackage{dsfont}

\usepackage{algorithm}%
\usepackage{algorithmicx}%
\usepackage{algpseudocode}%
\usepackage{subcaption}

\usepackage{amsthm}%

\newtheorem{remark}{Remark}
\newtheorem{definition}{Definition}
\newtheorem{lemma}{Lemma}

\usepackage{mathtools}
\mathtoolsset{showonlyrefs} 

\usepackage{comment}

\newcommand{\err}{\text{err}}
\newcommand{\Err}{\text{Err}}
\newcommand{\XX}{\boldsymbol{X}}
\newcommand{\YY}{\boldsymbol{Y}}
\newcommand{\ZZ}{\boldsymbol{Z}}
\newcommand{\eps}{\varepsilon}

\usepackage{todonotes}

\usepackage{amssymb, amsmath, amsfonts, stmaryrd}

\journal{Pattern Recognition}

\begin{document}

\begin{frontmatter}



\title{Beyond Conformal Predictors: Adaptive Conformal Inference with Confidence Predictors}


\author[inst1,inst2]{Johan Hallberg Szabadv\'ary\corref{cor1}}
\ead{johan.hallberg.szabadvary@math.su.se,johan.hallberg.szabadvary@ju.se}
\author[inst2]{Tuwe L\"ofstr\"om}
\ead{tuwe.lofstrom@ju.se}

\affiliation[inst1]{organization={Department of Mathematics, Stockholm University},
            city={Stockholm},
            country={Sweden}}

\affiliation[inst2]{organization={Department of Computing, Jönköping School of Engineering},
            city={Jönköping},
            country={Sweden}}

\cortext[cor1]{Corresponding author}

\begin{abstract}
    Adaptive Conformal Inference (ACI) provides finite-sample coverage guarantees, enhancing the prediction reliability under non-exchangeability. This study demonstrates that these desirable properties of ACI do not require the use of Conformal Predictors (CP). We show that the guarantees hold for the broader class of confidence predictors, defined by the requirement of producing nested prediction sets, a property we argue is essential for meaningful confidence statements. We empirically investigate the performance of Non-Conformal Confidence Predictors (NCCP) against CP when used with ACI on non-exchangeable data. In online settings, the NCCP offers significant computational advantages while maintaining a comparable predictive efficiency. In batch settings, inductive NCCP (INCCP) can outperform inductive CP (ICP) by utilising the full training dataset without requiring a separate calibration set, leading to improved efficiency, particularly when the data are limited. Although these initial results  highlight NCCP as a theoretically sound and practically effective alternative to CP for uncertainty quantification with ACI in non-exchangeable scenarios, further empirical studies are warranted across diverse datasets and predictors.
\end{abstract}

\begin{graphicalabstract}
\includegraphics[width=\textwidth]{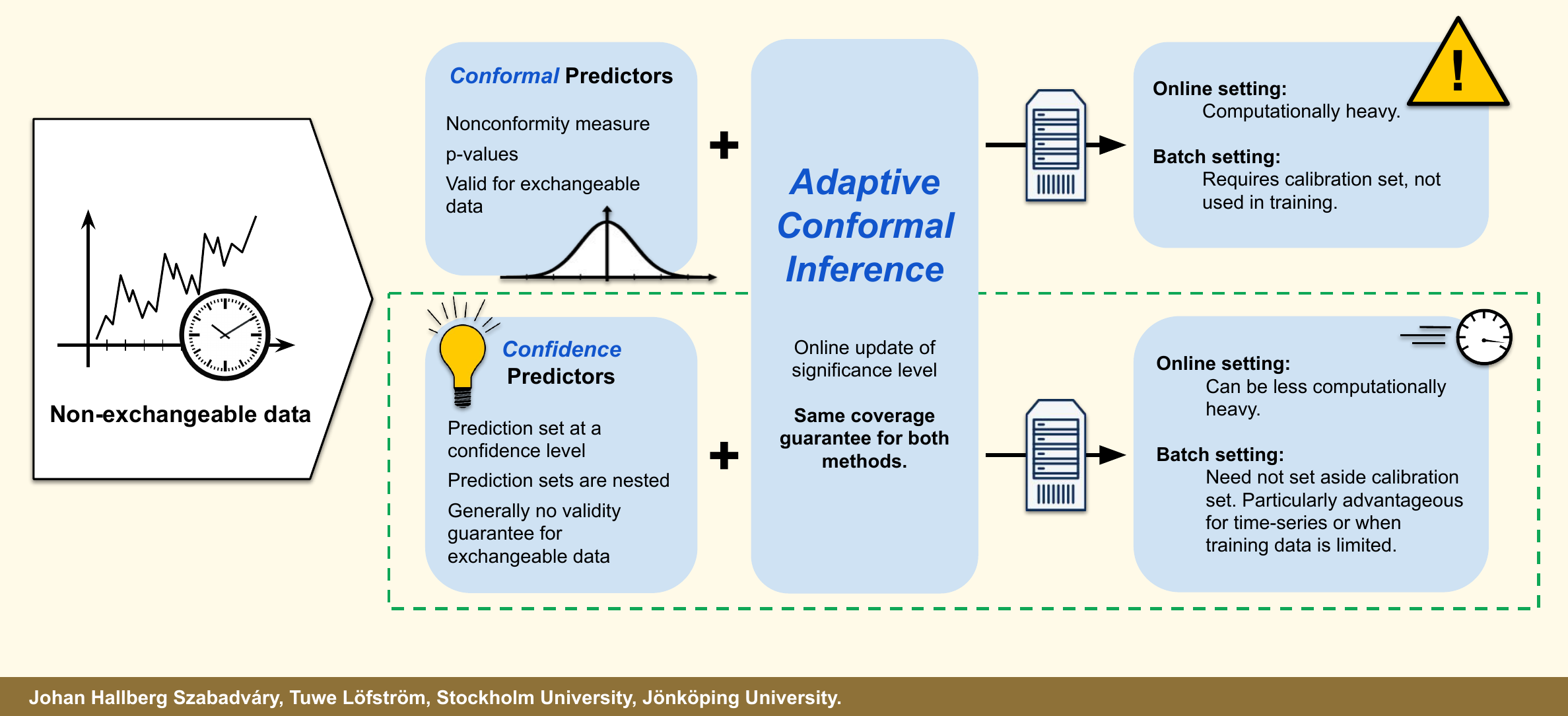}
\end{graphicalabstract}

\begin{highlights}
    \item 
    Shows that adaptive conformal inference (ACI) achieves finite-sample coverage guarantees without needing conformal predictors (CP).
    \item 
    Demonstrates empirically that non-conformal confidence predictors (NCCP) can offer comparable or better performance than CP for non-exchangeable data.
    \item 
    Highlights that inductive NCCP avoids sacrificing training data for calibration, potentially improving predictive efficiency compared to inductive CP, especially when data are limited.
    \item
    Clarifies that the primary advantage of NCCP in online settings may be computational efficiency, whereas in batch settings, it relates to improved data utilisation.
    \item 
    Argues that nested prediction sets are necessary for meaningful confidence-based predictions regardless of the underlying predictor type.
\end{highlights}


\end{frontmatter}



\section{Introduction}\label{sec:intro}

    The need for uncertainty quantification in pattern recognition arises in many safety-critical applications, such as medical diagnosis and autonomous driving, where recognising and addressing uncertainty can prevent costly or dangerous mistakes. Unfortunately, many models, including modern deep-learning methods, suffer from poor calibration \cite{guo2017calibration} and \cite{nguyen2015deep}, meaning that when the model is asked to predict with 80\% confidence, we cannot expect it to be correct 80\% of the time. 
    
    Conformal prediction (CP) is a general framework for distribution-free uncertainty quantification. It can be used with essentially any machine learning algorithm and has guaranteed properties of validity, provided that the data are drawn from a probability distribution that is exchangeable, meaning essentially that any permutation is equally probable. For details on exchangeability, see \cite{alrw2}. The strongest notions of validity are achieved in the online setting, where Reality outputs successive examples $z_n:=(x_n, y_n)\in \boldsymbol{Z}:=\boldsymbol{X}\times \boldsymbol{Y}$, each of which consists of an object $x_n\in \boldsymbol{X}$ and its associated label $y_n\in \boldsymbol{Y}$. CP produces prediction sets, $\Gamma_n^{\varepsilon}$, at a user-specified significance level $\varepsilon$, using $z_1,\dots, z_{n-1}$ and $x_n$ as inputs. 
    
    If the exchangeability assumption is violated, the validity guarantees of CP are lost. Adaptive Conformal Inference (ACI) was suggested by Gibbs and Cand\`es \cite{gibbs2021adaptive} as a method to produce prediction sets that are robust to violations of exchangeability.     
    Fundamentally, ACI employs a simple feedback control mechanism that dynamically adjusts the working significance level based on recent prediction errors to steer the long-term empirical error rate towards a desired target $\varepsilon$.
    In this study, we show that there is no need to use ACI with a conformal predictor. The same guaranteed error rates hold even if we use the more general notion of a confidence predictor, which is often much more computationally tractable in the online setting. The applicability of ACI to non-conformal set predictors follows directly from its proof, which makes it somewhat surprising that, to the best of our knowledge, it has not been highlighted in the literature. Since nothing is gained by using the heavier machinery of CP in terms of error rate, the question is whether CP offers more efficient prediction sets. 

    We investigate this question empirically, both in the natural habitat of CP, the online setting, and the more popular inductive or batch setting, where inductive conformal predictors (ICP) are more computationally manageable. Our numerical experiments on publicly available datasets that are known to be non-exchangeable indicate that CP adds little beyond the computational effort, when compared with non-conformal confidence predictors (NCCP), when the assumption of exchangeability is violated in the online setting. In the more popular batch setting, we compare ICP with inductive NCCP (INCCP). 
    One may suspect that sacrificing some of the available training data for calibration, which is required by ICP, may lead to less efficient (informally larger) prediction sets compared with alternative methods. This intuition is supported by our investigation of the efficiency over different calibration set sizes.

    The remainder of this paper is organised as follows. Section \ref{sec:related} discusses related work on the subject of uncertainty quantification in general, and CP for non-exchangeable data in particular. Section \ref{sec:conformalTheory} introduces conformal and confidence predictors as well as the relevant notation. Readers familiar with the terminology used in \cite{alrw2} may find it convenient to skip ahead to Section \ref{sec:ACI}, where ACI and its finite sample coverage guarantee is introduced, as well as a trivial but technically necessary extension of conformal and confidence predictors. Our main result is given in Section \ref{sec:nothingConformal}, where we restate Lemma 4.1 from \cite{gibbs2021adaptive}, which is the key result used to prove the finite sample guarantee of ACI. We indicate that this lemma does not rely on any particular property of conformal, or even confidence predictors. We then argue in Section \ref{sec:whyConfidencePredictors} that while not strictly necessary for the finite sample guarantee of ACI, confidence predictors represent the natural way of predicting at a confidence level. Numerical experiments on synthetic and real data are described in Section \ref{sec:experimentSetup}, and the results are presented in Section \ref{sec:results}. Finally, Section \ref{sec:conclusions} concludes the paper. 

\section{Related work}\label{sec:related}
    \subsection{Uncertainty quantification}
        The main reference for CP is the book “Algorithmic Learning in a Random World” \cite{alrw1} and its second edition \cite{alrw2}. Several excellent shorter introductions to CP exist, including \cite{TOCCACELI2022108507} and \cite{shafer2008tutorial}. CP has been the subject of special issues of several journals, including Pattern Recognition \cite{GAMMERMAN2022108561} and Neurocomputing \cite{GAMMERMAN2020264}. 
    
        The most basic form of CP solves the uncertainty quantification problem by, instead of predicting a single point, outputting a prediction set associated with a user-specified significance level $\eps$. Provided only that the data are exchangeable with no other distributional assumptions, the prediction sets will fail to contain the true label with probability $\eps$, making conformal predictors provably well calibrated \cite{alrw2}. Depending on the specific use case, set predictions may not be desirable; in this case, the conformal framework can be used to produce probabilistic predictions both in classification (Venn predictors \cite{vovk2014venn}) and regression (Conformal Predictive Systems \cite{VOVK2022108536}). 
    
        Bayesian methods \cite{ghahramani2015probabilistic} and \cite{YELLENI2024110003} provide an alternative uncertainty quantification which, however, rely on a well-specified prior. Conformal predictors can be used in conjunction with Bayesian models for protection. For details, see \cite{burnaev2014efficiency} and the discussion of hard and soft models in \cite{alrw2}.
    
        It is well known \cite{nguyen2015deep} and \cite{guo2017calibration} that many machine learning models are poorly calibrated, and several post-hoc calibration methods, including Platt scaling \cite{platt1999probabilistic}, temperature scaling \cite{guo2017calibration} and Dirichlet scaling \cite{kull2019beyond} exist to address this problem. Common for post-hoc calibration methods is to use a calibration set, not used in training, to learn a calibration map that transforms the model's predictions to better calibrated ones. 
    
    \subsection{Conformal prediction for non-exchangeable data}
        While powerful and general, the validity guarantees of CP rely on the exchangeability assumption, which is a limitation e.g. in time-series forecasting. To address these limitations, many authors have suggested solutions to retain at least some validity guarantees even when exchangeability is violated. These include introducing weights to attempt to deal with distribution shifts, such as \cite{tibshirani2019conformal} and \cite{barber2023conformal}, and what could be described as control theoretic approaches that attempt to achieve the desired error rate by varying the significance level. 
        
        ACI \cite{gibbs2021adaptive} falls into the second category of control theoretic approaches and has become a popular choice for non-exchangeable data. It has been suggested as a good method for time-series forecasting \cite{zaffran2022adaptive}. It was implemented in this mode in the Python package MAPIE \cite{Cordier_Flexible_and_Systematic_2023}. In the time-series domain, it has also been suggested for multi-step-ahead forecasting by \cite{sousa2024general} and \cite{pmlr-v230-hallberg-szabadvary24a}. The more complicated conformal PID algorithm \cite{angelopoulos2024conformal}, inspired by classical proportional-integral-derivative (PID) control, can be viewed as a generalisation of ACI, as the authors recover it as a special case of conformal PID. Recent developments include adapting the step-size parameter in ACI \cite{gibbs2024conformal} and \cite{pmlr-v235-podkopaev24a}.

\section{Theoretical background}\label{sec:conformalTheory}

    This section introduces the relevant theoretical background of mathematical objects and notation. Readers familiar with conformal prediction and the notation used in \cite{alrw2} may find it convenient to skip ahead to the next section, where ACI is introduced.
    
    \subsection{Confidence prediction in the online setting}
        Given two measurable spaces $\XX$ and $\YY$, the former called the object space and the latter is the label space, we assume that Reality outputs successive pairs
        \begin{equation}\label{eq:onlineExamples}
            (x_1,y_1),(x_2,y_2),\dots
        \end{equation}
        called examples. For notational convenience, we write $z_i := (x_i,y_i)\in\XX\times\YY:=\ZZ$. The measurable space $\ZZ$ is called the example space. Thus, the infinite sequence \eqref{eq:onlineExamples} is an element of the measurable space $\ZZ^{\infty}$ and we assume that it is drawn from a probability distribution $P$ on $\ZZ^{\infty}$. The standard assumption in CP is that $P$ is exchangeable, but in this study, we make no such assumption, and $P$ can be any distribution.

        Most machine learning methods are, so called, simple predictors, with the aim of predicting the label $y_n$. Here, we are interested in another type of prediction. Instead of predicting $y_n$, we want to predict subsets of $\YY$ of varying sizes, but large enough that we can be confident that $y_n$ will fall in them. Of course, this is a simple task; just predict $\YY$, and we are absolutely sure that the prediction set will contain $y_n$. However, we may be willing to accept a slightly smaller confidence level, provided that the prediction set is smaller and thus more informative. 
        \begin{definition}[Confidence predictor]
            \label{def:confPred}
            A confidence predictor is a measurable function 
            \begin{equation}
                \Gamma : \boldsymbol{Z}^*\times\boldsymbol{X}\times(0,1) \to 2^{\boldsymbol{Y}}
            \end{equation}
            that, for each significance level $\varepsilon\in(0,1)$ outputs the prediction set 
            \begin{equation}
                \Gamma_n^{\varepsilon} := \Gamma^{\varepsilon}(x_1,y_1,\dots,x_{n-1},y_{n-1},x_n)
            \end{equation}
            with the additional property that $\varepsilon_1\geq\varepsilon_2$ implies
            \begin{equation}
                \Gamma^{\varepsilon_1}(x_1,y_1,\dots,x_{n-1},y_{n-1},x_n) \subseteq \Gamma^{\varepsilon_2}(x_1,y_1,\dots,x_{n-1},y_{n-1},x_n).
            \end{equation}
            This last property is called the nested prediction set property.
        \end{definition}
        Recall that $2^{\YY}$ is the set of all subsets of $\YY$. Thus, a confidence predictor outputs a subset of the label space based on all previous examples and the current object, and the idea is that it should contain the true label $y_n$ with a user-specified confidence. Moreover, the prediction sets are nested, as illustrated in Figure \ref{fig:nestedSets}.
        \begin{figure}[h!]
            \centering
                \begin{tikzpicture}
                    \definecolor{lightgray}{gray}{0.8}
                    \definecolor{mediumgray}{gray}{0.6}
                    \definecolor{darkgray}{gray}{0.4}
                    
                    \filldraw[fill=lightgray] (0, 0) ellipse (4 and 2.5);
                    
                    \filldraw[fill=mediumgray] (-1, 0) ellipse (2.5 and 1.5);
                    
                    \filldraw[fill=darkgray] (-0.3, 0.2) ellipse (1.5 and 1);
                    
                    \draw[black] (4, 1.5) rectangle (6.5, 4);
                    \filldraw[fill=lightgray] (4.2, 3.7) rectangle (4.7, 3.2);
                    \node at (5.5, 3.45) {$\eps=0.1$};
                    \filldraw[fill=mediumgray] (4.2, 3) rectangle (4.7, 2.5);
                    \node at (5.5, 2.75) {$\eps=0.2$};
                    \filldraw[fill=darkgray] (4.2, 2.3) rectangle (4.7, 1.8);
                    \node at (5.5, 2.05) {$\eps=0.3$};
                \end{tikzpicture}
            \caption{Illustration of nested prediction sets.}
            \label{fig:nestedSets}
        \end{figure}
        
        Several machine-learning methods can be trivially turned into confidence predictors. In the classification case, confidence thresholding using confidence scores (for example, \cite{HAGHPANAH2023109683} or simple majority voting in ensemble methods are simple ways to define confidence predictors. For regression, many methods natively support confidence intervals. In the case of parametric methods, these are nested, such as ordinary least squares. Another method is quantile regression, but care must be taken, as many quantile regression methods do not guarantee that the prediction sets are nested. However, it is always possible to force prediction intervals to be nested by applying isotonic regression (see \cite{alrw2} or \cite{Ayer1955AnInformation} for details) to the predicted quantiles and use the calibrated predictive quantiles to construct prediction intervals.
    
        We need some more notation. Let $\Gamma$ be a confidence predictor that processes the data sequence
        \begin{equation}
            \omega = (x_1,y_1,x_2,y_2,\dots)
        \end{equation}
        at significance level $\varepsilon$. We say that $\Gamma$ makes an error in the $n$th trial if $y_n\notin\Gamma_n^{\varepsilon}$. More precisely,
        \begin{equation}\label{eq:err}
            \err_n^{\varepsilon}(\Gamma, \omega) := 
            \begin{cases}
                1 & \text{if $y_n\notin\Gamma_n^{\varepsilon}$} \\
                0 & \text{otherwise,}
            \end{cases}
        \end{equation}
        and the number of errors during the first $n$ trials is 
        \begin{equation}\label{eq:Err}
            \Err_n^{\varepsilon}(\Gamma,\omega) := \sum_{i=1}^{n}\err_n^{\varepsilon}(\Gamma, \omega).
        \end{equation}
    
    \subsection{Validity}
    
        The number $\err_n^{\eps}(\Gamma, \omega)$ is the realisation of a random variable $\err_n^{\eps}(\Gamma)$.
        We say that a confidence predictor $\Gamma$ is exactly valid if, for each $\eps$, 
        \begin{equation}
            \err_1^{\eps}(\Gamma), \err_2^{\eps}(\Gamma),\dots
        \end{equation}
        is a sequence of independent Bernoulli random variables with parameter $\eps$. In other words, the event of making an error is like getting heads when tossing a biased coin, where the probability of getting heads is always $\eps$. There is also the notion of conservative validity, which is more complicated to state but essentially means that the error sequence is dominated by a sequence of independent Bernoulli variables with parameter $\eps$. For the complete statement, we refer to \cite{alrw2}.
        
        We say that $\Gamma$ is asymptotically valid if 
        \begin{equation}\label{eq:assVal}
            \lim_{n\to\infty}\frac{1}{n}\Err_n^{\eps}(\Gamma) = \eps
        \end{equation}
        for each $\eps\in(0,1)$.
    
    \subsection{Conformal predictors}
    
        Conformal predictors rely on the notion of a nonconformity score. Informally, this is a function that quantifies how “strange” or “unusual” an example $z$ is in relation to what we have seen before. More formally, given a sequence of examples $z_1, \dots, z_n$, we form the bag (or multiset) $\Sigma_n := \lbag z_1,\dots,z_n\rbag$. The set of all bags of size $n$ formed from examples in $\ZZ$ is denoted $\ZZ^{(n)}$. A nonconformity measure is a measurable function 
        \begin{equation}
            \begin{aligned}
                A : \ZZ^{(*)}\times\ZZ &\to \overline{\mathbb{R}}\\
                (\Sigma, z) &\mapsto \alpha.                
            \end{aligned}
        \end{equation}
        Given a simple predictor, for example, a machine learning model that outputs the prediction $\hat{y}$ for the label $y$, a natural choice of the nonconformity measure is $\alpha := |y-\hat{y}|$.
        \begin{definition}[Conformal predictor]
            The conformal predictor determined by a nonconformity measure $A$ is the confidence predictor defined by setting
            \begin{equation}\label{eq:defCP}
                \Gamma^{\eps}(z_1,\dots,z_{n-1}, x_n):=\Gamma^{\eps}_n :=\left\{y\in \boldsymbol{Y}:\frac{|\{i=1,\dots,n:\alpha_i\geq\alpha_n\}|}{n} > \eps \right\}
            \end{equation}
            where
            \begin{equation}
                \begin{aligned}
                    \alpha_i &:= A(\Sigma,(x_i,y_i)), & i=1,\dots,n-1\\
                    \alpha_n &:= A(\Sigma,(x_n,y)), & \\
                    \Sigma &:= \Lbag (x_1,y_1),\dots,(x_{n-1},y_{n-1}), (x_n,y)\Rbag.
                \end{aligned}
            \end{equation}
            The fraction $\frac{|\{i=1,\dots,n:\alpha_i\geq\alpha_n\}|}{n}$ is called the p-value of the example $(x_n,y)$.
        \end{definition}  
                
        The key result about conformal predictors is that they are valid under the exchangeability assumption. A conformal predictor is conservatively valid, and by introducing some randomisation to the special cases when $\alpha_i=\alpha_j$, which results in a smoothed conformal predictor, we can even get exact validity. However, when exchangeability is violated, these validity guarantees are lost.
        
        The main disadvantage of conformal predictors is their, often, intractable computational cost. In a regression setting where $\YY=\mathbb{R}$, it is theoretically necessary to compute the p-value in \eqref{eq:defCP} for each real number, which is clearly impossible. For some special cases, such as ridge regression, efficient implementations exist, but generally the computational cost is too high. For this reason, inductive conformal predictors (ICP) have been introduced. We provide a brief description here and refer the interested reader to \cite{alrw2} for further details. Given a training set of size $l$, we split it into two parts: the proper training set $z_1,\dots, z_m$ of size $m$ and the calibration set $z_{m+1},\dots,z_l$ of size $l-m$. For every test object $x_i$, compute the prediction set
        \begin{equation}\label{eq:defICP}
            \Gamma^{\eps}(z_1,\dots,z_{l}, x_i):=\left\{y\in \boldsymbol{Y}:\frac{|\{j=m+1,\dots,l:\alpha_j\geq\alpha_i\}|+1}{l-m+1} > \eps \right\},
        \end{equation}
        where the nonconformity scores are 
        \begin{equation}
            \begin{aligned}
                \alpha_j&:=A((z_1,\dots,z_m),z_j), & j=m+1,\dots,l,\\
                \alpha_i &:= A((z_1,\dots,z_m),(x_i,y)).
            \end{aligned}
        \end{equation}
        These are the most widely used conformal predictors in practice, but as always, there is a price to pay for the computational efficiency. ICPs are no longer exactly valid, but their validity property is of the probably approximately correct (PAC) type, with two parameters. For details on the training-conditional validity of ICPs, see \cite{pmlr-v25-vovk12}.
        \begin{remark}
            Vapnik \cite{vapnik1999nature} distinguished between induction and transduction, as applied to the prediction problem as follows. In inductive prediction, we start from some observed examples and extract some more or less general prediction rule, model or theory. This is the inductive step. For a test object, we derive a prediction from this general rule, which is the deductive step. Transductive prediction skips the inductive step and moves directly from observed examples to the prediction for a new object.
        \end{remark}

        As definition \ref{eq:defCP} makes clear, conformal predictors are confidence predictors. From now on, we shall write non-conformal confidence predictors (NCCP) whenever we refer to confidence predictors that are not conformal predictors. An inductive confidence predictor that is not an inductive conformal predictor will be denoted by INCCP. The term confidence predictor refers to any confidence predictor, whether conformal or not.
    
\section{Adaptive conformal inference}\label{sec:ACI}

    Adaptive conformal inference \cite{gibbs2021adaptive} has been suggested as a method to achieve asymptotic validity under non-exchangeable data. The idea is that instead of using the same significance level for all predictions, we use the online update
    \begin{equation}
        \label{eq:ACI}
        \varepsilon_{n+1} = \varepsilon_n + \gamma(\varepsilon - \err^{\varepsilon_n}_n)
    \end{equation}
    where $\gamma$ is the step size (or learning rate). In other words, if we made an error in the last step, we increase the significance level; otherwise, we decrease it. Gibbs and Cand\`es proved the following finite sample guarantees of ACI.
    \begin{equation}\label{eq:ACIGuarantee}
        \bigg|\varepsilon - \frac{1}{N}\sum_{i=1}^N\err_{n}^{\varepsilon_{n}}(\Gamma) \bigg| 
        \leq
        \frac{\max\{\varepsilon_1, 1-\varepsilon_1\} + \gamma}{\gamma N} ~(a.s),
    \end{equation}
    which in particular converges to 0 as $N\to\infty$, ensuring asymptotic validity.
    Note that if we want a certain absolute error deviation bound, $\delta > \frac{\max\{\varepsilon_1, 1-\varepsilon_1\}+1}{N}$, for a desired error rate $\eps$, with a finite sample size $N$, we can almost surely achieve it by choosing
    \begin{equation}\label{eq:gammaChoice}
        \gamma = \frac{\max\{\varepsilon_1, 1-\varepsilon_1\}}{\delta N - 1}.
    \end{equation}

    The update rule \eqref{eq:ACI} resembles a simple controller that increases or decreases the significance level to counteract deviations from the target error rate, $\varepsilon$. 
    Importantly, the iteration \eqref{eq:ACI} can cause $\eps_n\leq0$ or $\eps_n\geq1$. Technically, prediction sets are undefined for confidence predictors for $\eps\notin(0,1)$. Thus, we introduce a trivial extension that is compatible with ACI.
    \begin{definition}[Extended confidence predictor]
        \label{def:extConfPred}
        An extended confidence predictor is a confidence predictor as defined in definition \ref{def:confPred}, with the additional property that
        \begin{equation}
            \begin{aligned}
                \varepsilon \leq 0 &\implies \Gamma^{\varepsilon}(x_1,y_1,\dots,x_{n-1},y_{n-1},x_n) = \boldsymbol{Y}\\
                \varepsilon \geq 1 &\implies \Gamma^{\varepsilon}(x_1,y_1,\dots,x_{n-1},y_{n-1},x_n) = \emptyset.
            \end{aligned}
        \end{equation}
    \end{definition}
    Since the difference between confidence predictors and extended confidence predictors is minor, and since any confidence predictor can be trivially modified to be an extended confidence predictor, we will use the terms interchangeably. 

\section{The use of conformal predictors is not essential}\label{sec:nothingConformal}
    Interestingly, the proof of \eqref{eq:ACIGuarantee} does not use any property of a conformal, or even confidence predictor. In fact, all that is required of a predictor $\Gamma$ is that $\emptyset$ is predicted for $\eps\geq1$ and $\YY$ for $\eps\leq0$. For all other significance levels, even completely random prediction sets ensure \eqref{eq:ACIGuarantee}.    
    The result is proved by Lemma 4.1 in \cite{gibbs2021adaptive}, which we reproduce here with a notation slightly modified to align with ours.
    \begin{lemma}[Lemma 4.1 in \cite{gibbs2021adaptive}]\label{lemma:ACI}
        For all $t\in\mathbb{N}$, $\varepsilon_n\in[-\delta,1+\delta]$, almost surely, if ACI \eqref{eq:ACI} is used together with a set predictor that outputs $\emptyset$ for $\eps\geq1$ and $\YY$ for $\eps\leq0$.
    \end{lemma}
    \begin{proof}
        Assume that the sequence $\{\varepsilon_n\}_{t\in\mathbb{N}}$ is such that $\inf_n\varepsilon_n<-\gamma$ (the case when $\sup_n\varepsilon_n>1+\gamma$ is identical). By \eqref{eq:ACI}, $\sup_n|\varepsilon_{n+1}-\varepsilon_n| = \sup_n\gamma|\varepsilon-\err_n^{\varepsilon_n}|<\gamma$. Thus, with positive probability, we may find $t\in\mathbb{N}$ such that $\varepsilon_n<0$ and $\varepsilon_{n+1}<\varepsilon_n$. However, by assumption, and \eqref{eq:ACI},
        \begin{equation}
            \varepsilon_n<0 \implies \Gamma_n^{\varepsilon_n} = \boldsymbol{Y} \implies \err_n^{\varepsilon_n} = 0 \implies \varepsilon_{n+1} = \varepsilon_n + \gamma(\varepsilon - \err_n^{\varepsilon_n}) \geq \varepsilon_n.
        \end{equation}
        We have reached a contradiction.
    \end{proof}
    It is clear that no property of conformal, or even confidence predictors, is required to achieve \eqref{eq:ACIGuarantee}. As Lemma \ref{lemma:ACI} and its proof demonstrate, the finite-sample guarantee \eqref{eq:ACIGuarantee} holds, provided that the predictor respects the boundaries ($\emptyset$ for $\varepsilon\geq1$ and $\boldsymbol{Y}$ for $\varepsilon\leq0$), and the significance level is updated via \eqref{eq:ACI}. This guarantee arises solely from the behaviour of the iterative error control loop.

\subsection{Why nested prediction sets matter}\label{sec:whyConfidencePredictors}
    This section argues that while not strictly necessary for achieving the finite sample guarantee \eqref{eq:ACIGuarantee} of ACI, we should restrict ourselves to confidence predictors, as this is the most general type of predictor that satisfies natural desiderata for prediction sets at a confidence level.
    
    \subsubsection{Validity is easy without nested prediction sets}

        We have seen that the same guarantee holds for ACI used with a predictor that outputs random subsets of $\YY$ for $\eps\in(0, 1)$. This predictor is not a confidence predictor because we have dropped the requirement for nested prediction sets. However, if we are willing to dispense with nested prediction sets, we could do even better. The coin-flip predictor outputs either $\emptyset$ or $\YY$ for $\eps\in(0, 1)$. Which one is determined by flipping a biased coin.
        \begin{definition}[coin-flip predictor]\label{def:coinFlipPredictor}
            A coin-flip predictor is a set predictor that for each significance level $\varepsilon\in(0,1)$ outputs the prediction set
            \begin{equation}
                \Gamma_n^{\varepsilon} := 
                \begin{cases}
                    \emptyset, & \text{with probability $\epsilon$} \\
                    \boldsymbol{Y} & \text{with probability $1-\epsilon$}.
                \end{cases}
            \end{equation}
        \end{definition}
        The point is that the coin-flip predictor is exactly valid. Its error sequence is the realisation of a sequence of independent Bernoulli variables with parameter $\varepsilon$ for each $\varepsilon\in(0,1)$. However, this is clearly not very informative, and we should be hesitant to use it in practice.

    \subsubsection{Conflicting predictions}

        We argue that if we want to predict sets with confidence, the very least we can ask for is a confidence predictor whose prediction sets are nested, as illustrated in Figure \ref{fig:nestedSets}. 
        Imagine a model tasked with diagnosing patients based on symptoms, such as cough, shortness of breath, fever, and chest pain. The model outputs prediction sets at confidence levels $0.7, 0.9$ and $0.99$, and the possible diagnoses are 
        \begin{itemize}
            \item 
            Healthy (no diagnosis)
            \item 
            Pneumonia,
            \item
            Bronchitis,
            \item 
            Asthma,
            \item 
            Chronic Obstructive Pulmonary Disease (COPD).
        \end{itemize}
        Suppose that, for some patient, the model outputs the prediction sets in Table \ref{tab:non-nested}.
        \begin{table}[h!]
            \centering
            \begin{tabular}{cc}
                \hline
               Confidence   & Prediction set \\
               \hline
                0.99        & \{Healthy, Pneumonia, Bronchitis, Asthma\}\\
                0.9         & \{Pneumonia, Bronchitis, Asthma\}\\
                0.7         & \{Healthy, COPD\}
            \end{tabular}
            \caption{A model that predicts with confidence, but does not have nested prediction sets.}
            \label{tab:non-nested}
        \end{table}
        
        The presented scenario illustrates a logical inconsistency in the model predictions. While asserting 90\% confidence that the patient does not have a healthy diagnosis, the model simultaneously includes "Healthy" within its 99\% confidence prediction set. Furthermore, at the 70\% confidence level, "Healthy" reappears alongside "COPD." However, as the confidence level increases to 90\% and subsequently 99\%, "COPD" is excluded. Given the potentially serious nature of these diagnoses, conflicting information from the model is problematic. Consequently, to ensure meaningful interpretation of confidence levels, the property of nested prediction sets is essential. Thus, we argue that if we are to mean anything reasonable by stating a confidence level, we simply must require that the prediction sets are nested.
        Therefore, while ACI does not strictly require it, its usage should be restricted to extended confidence predictors in the sense of Definition \ref{def:extConfPred}. However, this raises several  empirical questions.
        
    \subsection{Do we gain anything from conformal predictors?}
    
        We have shown that ACI achieves the finite sample guarantee \eqref{eq:ACIGuarantee} for non-exchangeable data, even for non-confidence predictors, but we have argued that it would be absurd to drop the property of nested prediction sets. A natural question is whether using ACI with a CP has any advantages over using it together with an NCCP. For non-exchangeable data, CP is not necessarily valid, which is also true for NCCP. 

        The observation that ACI achieves the same finite sample guarantee for any confidence predictor, be it CP or NCCP, raises at least two questions regarding its application. 
        \begin{enumerate}
            \item 
            In the online setting, can we use ACI with an NCCP instead of a CP? The validity guarantee of CP is lost anyway, so the question is whether CP is more efficient (in the sense of outputting smaller prediciton sets) than NCCP.
            \item 
            In the offline setting, ICP requires sacrificing some training data for calibration. If the data are not exchangeable, does it even help to do so? Since the ACI finite sample coverage guarantee holds regardless, might it be better to use all available data for training and NCCP to be as efficient as possible?
        \end{enumerate}
        We now proceed with our experiments, with the aim to answer these questions.

\section{Experimental setup}\label{sec:experimentSetup}
    Attempting to answer the questions raised by the observation that ACI achieves the same guarantee \eqref{eq:ACIGuarantee} for NCCP as for CP, we conduct four numerical experiments on publicly available datasets, known to be nonexchangeable.

    \begin{itemize}
        \item 
        The Wine Quality dataset is publicly available in the UCI Machine Learning repository \cite{wine_quality_186}. It consists of 11 features that may be useful for predicting the quality of wine, which is encoded as an integer between 3 and 9, inclusive. Most labels are between 4 and 7. There are 6497 examples in total, 4898 white and 1599 red. The exchangeability of the wine dataset was studied by \cite{pmlr-v152-vovk21b}, where it was shown that red and white wines are not exchangeable. 
        \item 
        The US Postal Service (USPS) dataset consists of 7291 training images and 2007 test images of handwritten digits from 0 to 9. The images are $16\times16$ grayscale pixels. The exchangeability of the USPS dataset was studied in \cite{vovk2003testing} and\cite{fedorova2012plug}. It is well known that the examples are not perfectly exchangeable. 
    \end{itemize}

    \subsection{Online experiments}
        Our online experiments attempt to answer the question of whether CP is more efficient than NCCP when used together with ACI in the online setting for nonexchangeable data. Recall that the ACI guarantee \eqref{eq:ACIGuarantee} holds for both CP and NCCP. 

        \subsubsection{Online classification}
            We process the USPS dataset in the original order (which is known to be non-exchangeable) using the first 100 examples as the initial training set. The desired error rate is $\varepsilon=0.1$. We use the following two confidence predictors:
            \begin{itemize}
                \item 
                \textbf{Nearest Neighbours Conformal Predictor (CP)}. The $k$-nearest neighbours CP is specified by the nonconformity measure
                \begin{equation}
                    A(\sigma, (x,y)) = \frac{\frac{1}{k} \sum_{x_i \in N_k(x, \sigma): y_i = y} d(x, x_i)}{\frac{1}{k} \sum_{x_j \in N_k(x, \sigma): y_j \neq y} d(x, x_j)}
                \end{equation}
                where $\sigma = \Lbag (x_1,y_1),\dots,(x_{n},y_{n})\Rbag \backslash \Lbag (x,y)\Rbag$, $d$ is the Euclidean distance, and $N_k(x, \sigma)$ denotes the set of the $k$ nearest neighbours of $x$ in $\sigma$ according to $d$ (or more generally, any distance). 
                \item 
                \textbf{Nearest Neighbours Confidence Predictor (NCCP)}. The $k$-nearest neighbours confidence predictor uses the confidence score $S_k(x, y) = N_k(x, y)/k$ to construct prediction sets. Here, $N_k(x, y)$ is the number of examples among the $k$ nearest neighbours of $x$ with label $y$. The prediction set is constructed using confidence thresholding: $\Gamma^{\varepsilon}(x) = \{y\in\YY : S_k(x, y) > \varepsilon\}$.
            \end{itemize}            
            With 9198 test examples, and $\varepsilon_1=\varepsilon$ for simplicity, we aim for an absolute error deviation bound $\delta = 0.01$, which, according to \eqref{eq:gammaChoice} means that we should choose the step size $\gamma \approx 0.001$.

            The $k$-nearest neighbours predictor produces prediction sets by finding the $k$ nearest neighbours of the test point $x_n$ in the training data, which can be performed in time $\mathcal{O}(n)$. For the $k$-nearest neighbours conformal predictor, nonconformity scores must be computed for each example in the bag, which takes time $\mathcal{O}(n^2)$, followed by computing the p-values for each potential label which takes time $\mathcal{O}(n)$ per label. 

        \subsubsection{Online regression}\label{sec:onlineRegressionSetup}
            Predicting the labels in the wine quality dataset can be seen as either regression or classification. We choose to perform the regression task again with the desired error rate $\eps=0.1$ and using the first 100 examples as the initial training set. With 6397 test examples and $\varepsilon_1=\varepsilon$ for simplicity, we can achieve a guaranteed absolute coverage deviation $\delta = 0.01$ if we choose $\gamma \approx 0.014$. We use the following confidence predictors:
            \begin{itemize}
                \item 
                \textbf{Conformalized Least Squares (CP)}. We use the conformalized ridge regression algorithm \cite{alrw2} with ridge parameter set to 0 (which corresponds to least squares). The conformalised ridge regression algorithm is described in Algorithm \ref{alg:crr}.
                \item 
                \textbf{Least Squares Condfidence Predictor (NCCP)}. The ordinary least-squares algorithm can output prediction intervals natively. We implement an online ordinary least squares algorithm and output the least squares prediction intervals. The resulting confidence predictor is described in \ref{alg:ridgeConfPred}.
            \end{itemize}
            The least-squares method goes back to Gauss and Legendre and is a widely used regression method. Ridge regression is a generalisation of the basic procedure, dating back to the 1960s, which introduces a non-negative ridge parameter $a$ (setting $a=0$ recovers ordinary least squares). In matrix form, it can be represented as
            \begin{equation}
                \label{RR}
                \omega = (X_n^TX_n+aI_p)^{-1}X_n^TY_n,
            \end{equation}
            where $Y_n := (y_1,\dots,y_n)^T$ and $X_n = (x_1,\dots,x_n)^T$, and the ridge regression prediction for an object $x$ is $\hat{y}:=x^T\omega$.  
            
            From Eq. \eqref{RR} we see that the predictions $\hat{y}$ for all objects $x_i$ are given by
            \begin{equation}
                \hat{Y}_n:=(\hat{y}_1,\dots,\hat{y}_n)^T = X_n(X_n^TX_n+aI_p)^{-1}X_n^TY_n,
            \end{equation}
            or if we define the hat matrix (because it maps $y_i$ to its “hatted” form)
            \begin{equation}
                \label{hatMatrix}
                H_n = X_n(X_n^TX_n+aI_p)^{-1}X_n^T,
            \end{equation}
            we can write $\hat{y} = H_nY_n$.
        
            The conformalised ridge regression algorithm (CRR) is a combination of two algorithms: lower CRR and upper CRR, which produce lower and upper bounds, respectively. For the upper CRR, we use the nonconformity measure $y-\hat{y}$, and for the lower CRR, we use $\hat{y}-y$, where $\hat{y}$ is the ridge regression prediction of $y$. The complete algorithm is shown in Algorithm \ref{alg:crr}. 
                            
            \begin{algorithm}
                \caption{Conformalised ridge regression (Alg. 2.4 \cite{alrw2})}\label{alg:crr}
                \begin{algorithmic}
                    \Require Ridge parameter $a\geq0$, significance level $\varepsilon\in(0,1)$, training set $(x_i,y_i)\in\mathbb{R}^p\times\mathbb{R}, i=1,\dots n-1$ and a test object $x_n\in\mathbb{R^p}$.
                    \State set $C=I_n-H_n$, $H_n$ being defined in Eq. \eqref{hatMatrix}
                    \State set $A = (a_1,\dots,a_n)^T:=C(y_1,\dots,y_{n-1},0)^T$
                    \State set $B = (b_1,\dots,b_n)^T:=C(0,\dots,0,1)^T$
                    \For{$i=1,\dots,n-1$}
                    \If{$b_n > b_i$} 
                    \State set $u_i:=l_i:=(a_i-a_n)/(b_n-b_i)$
                    \Else 
                    \State set $l_i=-\infty$ and $u_i=\infty$
                    \EndIf
                    \EndFor
                    \State sort $u_1,\dots,u_{n-1}$ in the ascending order obtaining $u_{(1)}\leq\dots\leq u_{(n-1)}$
                    \State sort $l_1,\dots,l_{n-1}$ in the ascending order obtaining $l_{(1)}\leq\dots\leq l_{(n-1)}$
                    \State output $[l_{(\lfloor (\varepsilon/2)n \rfloor)}, u_{(\lceil (1-\varepsilon/2)n \rceil)}]$ as prediction set.
                \end{algorithmic}
            \end{algorithm}
            In our experiments on synthetic data, we trivially extend Algorithm \ref{alg:crr} by requiring the output to be $(-\infty,\infty)$ for $\eps\leq0$ and $\emptyset$ for $\eps\geq1$, resulting in an extended conformal predictor.
            It was shown in \cite{alrw1} that the prediction sets can be computed in $\mathcal{O}(n\ln n)$ in the online mode. Computing $A$ and $B$ can be done in time $\mathcal{O}(n)$, and sorting can be done in time $\mathcal{O}(n\ln n)$.
        
            The confidence predictor based on ridge regression is summarised in Algorithm \ref{alg:ridgeConfPred}, where $t_{\eps/2,n-1-p}$ is the critical value from the Student's $t$-distribution with $n-1-p$ degrees of freedom.  
            
            \begin{algorithm}
                \caption{Ridge confidence predictor}\label{alg:ridgeConfPred}
                \begin{algorithmic}
                    \Require Ridge parameter $a\geq0$, significance level $\varepsilon\in(0,1)$, training set $(x_i,y_i)\in\mathbb{R}^p\times\mathbb{R}, i=1,\dots n-1$ and a test object $x_n\in\mathbb{R^p}$.
                    \State set $\hat{y}_n = x_n^T(X_{n-1}^TX_{n-1} + aI_p)^{-1}Y_{n-1}$
                    \State set $C = I_{n-1} - H_{n-1} = (c_1,\dots,c_{n-1})$, $H_n$ being defined in Eq. \eqref{hatMatrix}
                    \State set $\sigma^2 = \frac{1}{n - 1 -p}\sum_{i=1}^{n-1}c_i^2$
                    \State output $[\hat{y}_n - t_{\eps/2, n-1-p}, [\hat{y}_n  t_{\eps/2, n-1-p}]$ as prediction set.
                \end{algorithmic}
            \end{algorithm}
            Because Algorithm \ref{alg:ridgeConfPred} avoids the sorting required in Algorithm \ref{alg:crr}, it can be computed in $\mathcal{O}(n).$

    \subsection{Offline experiments}
        In the offline setting, there are no obvious computational gains from switching from ICP to INCCP. However, one might suspect that using all available data for training, rather than sacrificing some of them for calibration, could lead to some improvements. Under exchangeability, setting aside a calibration set is motivated because it guarantees validity. This is no longer the case for nonexchangeable data. Our offline experiments aims to investigate the effect of the calibration set size for ICP and comparing their efficiency (informally tight prediction sets) to that of INCCP, which requires no calibration. 

        \subsubsection{Offline classification}
            Again, we use the USPS dataset, but this time we use the 7291 training examples to train the following inductive confidence predictors.
            \begin{enumerate}
                \item 
                \textbf{Random forest inductive conformal predictor (ICP)}. This is a standard ICP defined using the \texttt{crepes} python package \citep{crepes}. The underlying predictor is the \texttt{RandomForestClassifier} from the \texttt{scikit-learn} python package \citep{scikit-learn} with default parameters; for example, the forest consists of 100 trees.
                \item 
                \textbf{Random forest inductive confidence predictor (INCCP)}. Based on the same \texttt{RandomForestClassifier}, the INCCP uses confidence thresholding using \\
                \texttt{predict\_proba}, computed by averaging the class probabilities predicted by each individual decision tree in the forest as confidence scores.
            \end{enumerate}
            Since the ACI update \eqref{eq:ACI} is inherently an online procedure, we predict the label of one test object at a time, observe its label, and update the significance level. However, we do not learn the example but use the same prediction rule for the entire test set. We run 1000 independent experiments using different random splits for the proper training and calibration sets. The random seed of each run is also passed to the \texttt{RandomForestClassifier}. For each seed, we run 99 experiments with calibration set sizes ranging from 1\% to 99\% of the training data, increasing in 1\% increments. The NCCP is run only once per random seed because it does not require any calibration set at all.

            With 2007 test examples, and  again choosing $\varepsilon_1=\varepsilon$, we aim for an absolute error deviation bound $\delta = 0.01$, which, according to \eqref{eq:gammaChoice} means that we should choose the step size $\gamma \approx 0.047$.

        \subsubsection{Offline regression}
            We use 4898 white wines as the training set and 1599 red wines as the test set to train the following confidence predictors.
            \begin{itemize}
                \item 
                \textbf{Random forest inductive conformal predictor (ICP)}. We use the \\
                \texttt{RandomForestQuantileRegressor} with default parameter from \cite{Johnson2024} as underlying predictor, which we turn into an ICP using \texttt{crepes}.
                \item 
                \textbf{Quantile forest confidence predictor (INCCP)}. The \\
                \texttt{RandomForestQuantileRegressor} is itself a confidence predictor, as it is shown in \cite{Johnson2024} that the predicted quantiles are monotinc by construction. Prediction intervals are thus constructed by setting the predicted $\varepsilon_n/2$ and $1-\varepsilon_n/2$ quantiles as endpoints.
            \end{itemize}
            The experiments are conducted using different calibration set sizes in the same manner as described for the offline classification experiment. Because we have 1599 test examples, again with $\varepsilon_1=\varepsilon$ and $\delta=0.01$, we take $\gamma \approx 0.06$ according to \eqref{eq:gammaChoice}.

    \subsection{Evaluation}\label{sec:evaluation}
        Efficiency criteria for CP often focus on the set size and/or the size of the p-values. This is possible because CP prediction sets are known to be valid under the assumption of exchangeability. Thus, the only consideration for evaluation is efficiency. Informally, an efficient CP tends to produce small prediction sets. From the perspective of conformal transducers, an efficient conformal transducer tends to output small p-values. For a full discussion of CP efficiency criteria, see \cite{vovk2016criteria} and Chapter 3 in \cite{alrw2}.
        
        For general confidence predictors (including CP for non-exchangeable data), validity cannot be assumed, which complicates the evaluation. In this setting, the desiderata are both calibration and efficiency. Informally, the prediction sets should be small and contain the true label with a probability approximately equal to the significance level $\varepsilon$ for all $\varepsilon\in(0,1)$. Crucially, proper scoring rules \cite{gneiting2007strictly} are designed such that their expected value is optimised when the forecaster's predictions perfectly match the true underlying distribution, thereby providing an incentive for honest and accurate uncertainty quantification. When we cannot expect validity, such as for non-exchangeable data, we therefore prefer to evaluate methods using proper scoring rules.

        For all experiments, we also report the empirical error rate to ensure that the ACI absolute error deviation bound \eqref{eq:ACIGuarantee} holds.

        \subsubsection{Regression}
            For interval predictions, where the output is the $1-\varepsilon$ prediction interval $[l, u]$, a popular proper scoring rule is the Winkler interval score, or just interval score \cite{winkler1972decision}
            \begin{equation}
                \label{eq:WinklerIntervalScore}
                S_{\varepsilon}^{\text{int}}(l, u, y) = (u-l) + \frac{2}{\varepsilon}(l-y)\mathds{1}_{\{y<l\}} + \frac{2}{\varepsilon}(y-u)\mathds{1}_{\{u<y\}}.
            \end{equation}
            The Winkler score is a proper scoring rule. It penalises intervals and incurs a penalty for misscoverage that depends on the significance level $\varepsilon$. Note that, assuming $\Gamma^{\varepsilon}$ is an interval, the Winkler interval score can be written
            \begin{equation}
                \label{eq:WinklerIntervalScoreGamma}
                S_{\varepsilon}^{\text{int}}(\Gamma_{\varepsilon}, y) = |\Gamma^{\varepsilon}| + \frac{2d(\Gamma^{\varepsilon}, y)}{\varepsilon},
            \end{equation}
            where $d(A, \xi) = \inf\{d(\xi, a):{a\in A}\}$ denotes the Euclidean distance from the point $\xi\in\mathbb{R}$ to the set $A\subset\mathbb{R}$.
    
            The mean Winkler interval score is the mean value of the individual interval scores. Smaller values indicate better performance. Since there is a possibility of outputting infinite intervals, we compute interval scores only for finite prediction intervals and report the fraction of infinite intervals separately. 

            Finally, the average width of the intervals is also reported because it is readily interpretable. When the exchangeability assumption is violated, we should prefer to evaluate using proper scoring rules, such as the Winkler score.
        
        \subsubsection{Classification}
        
            Unfortunately, to the best of our knowledge, no proper scoring rules for prediction sets in classification tasks exist in the literature. 
            Thus, in the classification setting, we report the observed excess (OE) of the prediction sets \cite{alrw2}, which is defined as
            \begin{equation}\label{eq:OE}
                OE(\Gamma, \eps) := \frac{1}{n}\sum_{i=1}^n|\Gamma^{\eps}_i\backslash\{y_i\},
            \end{equation}
            that is, the average number of false labels included in the prediction sets at significance level $\eps$. In contrast to the set size criterion, OE never penalises the correct label, which aligns with the preference for small, well-calibrated prediction sets. It is clear that smaller values are preferable. It was shown in \cite{alrw2} that the observed excess is a conditionally proper efficiency criterion, which in essence means that the true conditional probability of the label, given the object, is always an optimal conformity score with respect to the conditionally proper efficiency criterion. In this sense, a conditionally proper efficiency criterion is analogous to a proper scoring rule. However, as all efficiency criteria for CP, it assumes validity, which makes it less than ideal for us. Nevertheless, lacking proper scoring rules for prediction sets in classification, we settle for the OE, with a slight modification to account for the ACI update \eqref{eq:ACI}. We want to achieve the finite-sample guarantee \eqref{eq:ACIGuarantee} with a target error rate $\eps$. Then 
            \begin{equation}\label{eq:OEACI}
                OE_{ACI}(\Gamma, \eps):= \frac{1}{n}\sum_{i=1}^n|\Gamma^{\eps_i}_i\backslash\{y_i\},
            \end{equation}
            The difference is that the prediction set at step $i$ is computed with the significance level $\eps_i$. This is a natural modification, and we will still refer to it as the OE. 

\section{Results}\label{sec:results}

    \subsection{Online classification}
        The cumulative results of online classification on the USPS dataset are shown in Figure \ref{fig:onlineUSPS}, and the mean results are listed in Table \ref{tab:onlineUSPS}. The first thing to note is that both CP and NCCP achieve the ACI finite sample guarantee \eqref{eq:ACIGuarantee}. Recall that we choose $\gamma \approx 0.001$ and $\varepsilon_1 = \varepsilon=0.1$ to obtain the bound $\delta=0.01$. From Figure \ref{fig:onlineUSPSOE}, it is clear that CP is somewhat more efficient in terms of OE, which is also reflected in the fact that NCCP is slightly conservative, as shown in Figure \ref{fig:onlineUSPSerr}. However, in Table \ref{tab:onlineUSPS}, we see that the average OE for both methods is small and quite similar. 
        \begin{figure}[h!]
            \centering
            \begin{subfigure}[b]{0.48\linewidth}
                \includegraphics[width=\linewidth]{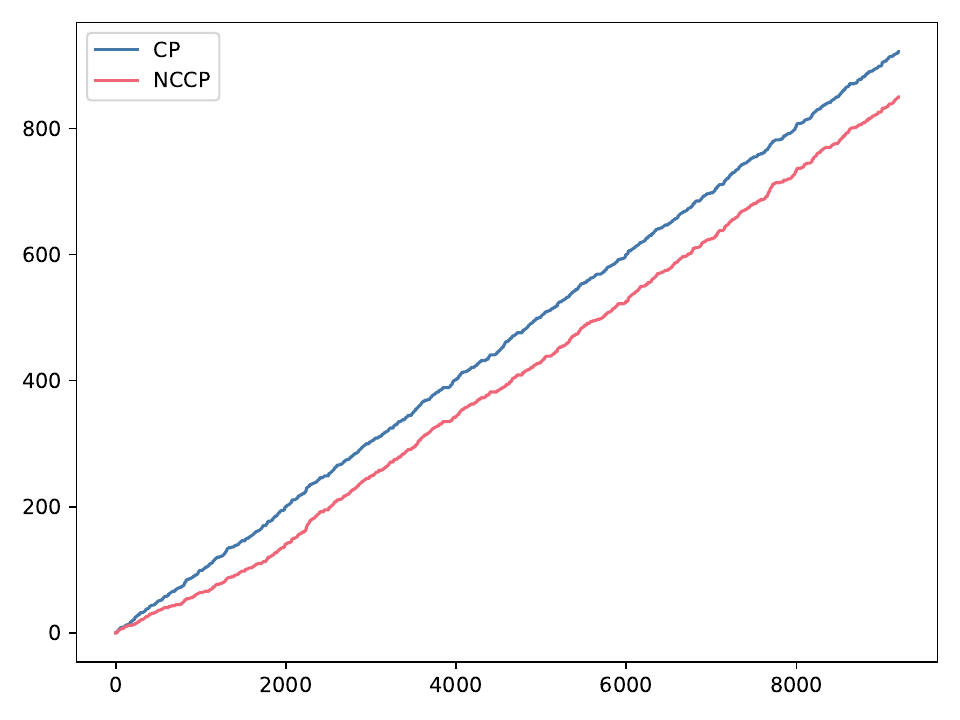}
                \caption{Cumulative error for CP and NCCP}
                \label{fig:onlineUSPSerr}
            \end{subfigure}
            \hfill 
            \begin{subfigure}[b]{0.48\linewidth}
                \includegraphics[width=\linewidth]{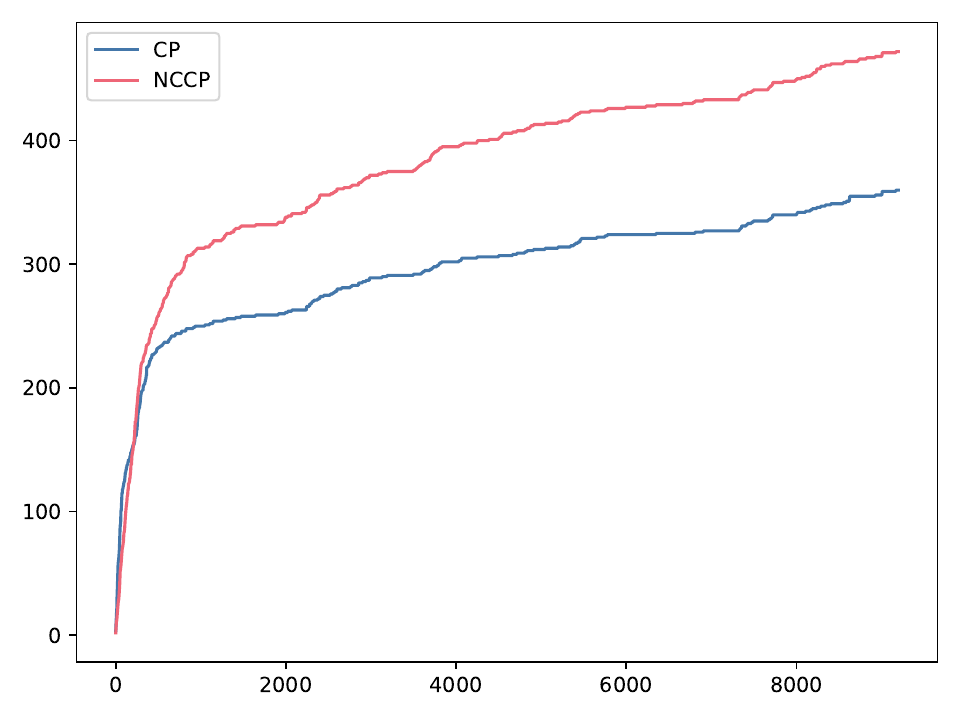}
                \caption{Cumulative observed excess (OE) for CP and NCCP}
                \label{fig:onlineUSPSOE}
            \end{subfigure}
            \caption{Cumulative results for CP and NCCP for online processing of the USPS dataset.}
            \label{fig:onlineUSPS}
        \end{figure}
        \begin{table}[h!]
            \centering
            \begin{tabular}{c|cc}
                Method & err & OE  \\
                \hline
                CP & 0.100239 & 0.039139 \\
                NCCP & 0.092411 & 0.051316
            \end{tabular}
            \caption{Mean error (err) and observed excess (OE) for CP and NCCP for online processing of the USPS dataset.}
            \label{tab:onlineUSPS}
        \end{table}

    \subsection{Online regression}
        The cumulative results of the online regression on the Wine dataset are shown in Figure \ref{fig:onlineWine} and Table \ref{tab:onlineWine}, where we also report the fraction of infinite intervals output by each method. Again, it should be noted that the finite-sample ACI guarantee \eqref{eq:ACIGuarantee} is achieved by both methods. We chose $\gamma \approx 0.014$ and $\varepsilon_1=\varepsilon$ to achieve the bound $\delta=0.01$. Interestingly, the cumulative error and Winkler score in Figures \ref{fig:onlineWineErr} and \ref{fig:onlineWineWinkler} are basically indistinguishable, but there are some small deviations in the cumulative width in Figure \ref{fig:onlineWineWidth}, which indicates that when the NCCP intervals are wider than those of CP, this is needed for coverage. If two intervals have the same Winkler score but one is wider, this means that the narrow interval does not cover the true label (see \eqref{eq:WinklerIntervalScoreGamma}). The average error, Winkler score, and width, as well as the fraction of infinite intervals, are listed in Table \ref{tab:onlineWine}. It is clear that NCCP is slightly better, as measured by the Winkler score, but also that CP outputs fewer infinite intervals, which is an advantage. Because the mean Winkler score is computed considering only finite intervals, it is not obvious which method is better, and in any case, the difference is small by any metric. Apart from the computational cost discussed in Section \ref{sec:onlineRegressionSetup}, our results show little difference between CP and NCCP for the Wine dataset.
        \begin{figure}[h!]
            \centering
            \begin{subfigure}[b]{0.48\linewidth}
                \includegraphics[width=\linewidth]{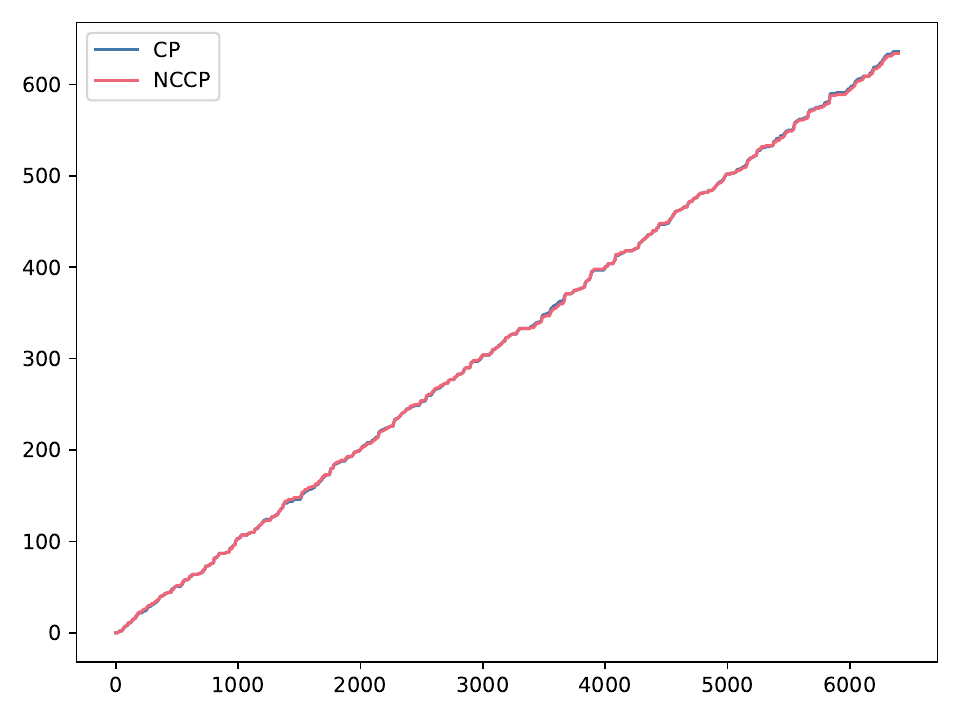}
                \caption{Cumulative error for CP and NCCP}
                \label{fig:onlineWineErr}
            \end{subfigure}
            \hfill 
            \begin{subfigure}[b]{0.48\linewidth}
                \includegraphics[width=\linewidth]{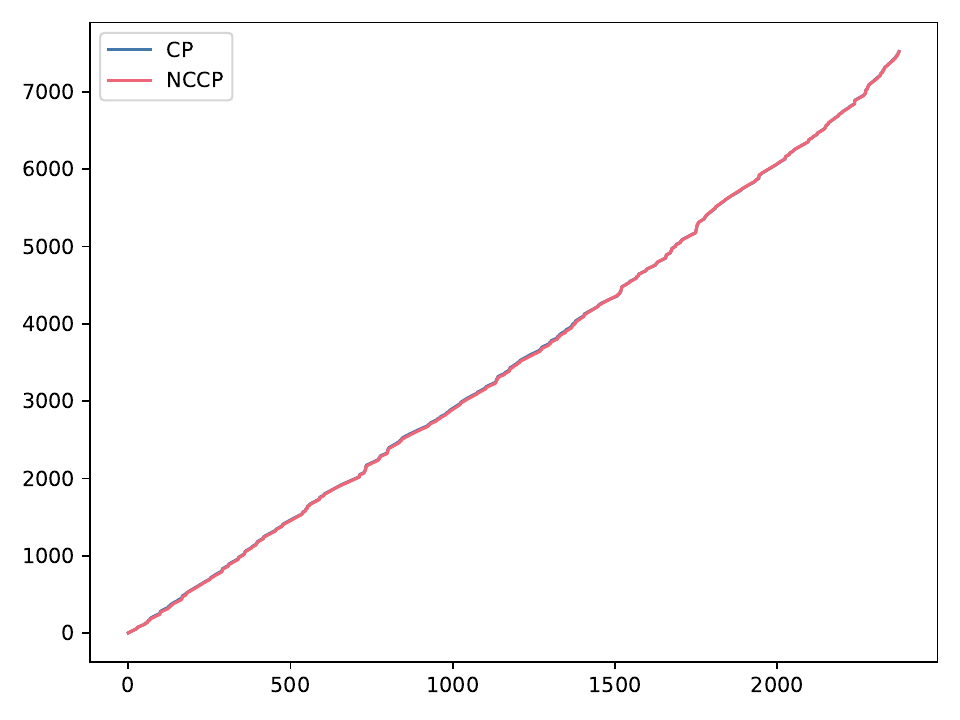}
                \caption{Cumulative Winkler score for CP and NCCP}
                \label{fig:onlineWineWinkler}
            \end{subfigure}
        
            \bigskip 
        
            \begin{subfigure}[b]{\linewidth} 
                \centering
                \includegraphics[width=0.48\linewidth]{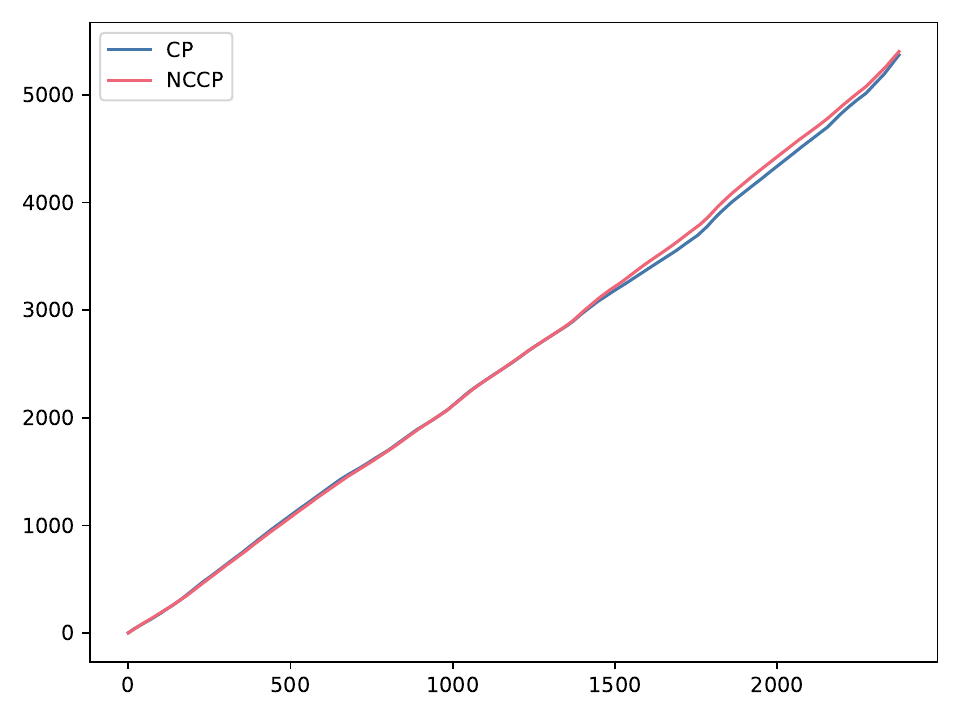}
                \caption{Cumulative width for CP and NCCP}
                \label{fig:onlineWineWidth}
            \end{subfigure}
            \caption{Cumulative results for CP and NCCP for online processing of the Wine dataset. Note that infinite intervals are not included in the cumulative sums.}
            \label{fig:onlineWine}
        \end{figure}
        \begin{table}[h!]
            \centering
            \begin{tabular}{c|cccc}
                Method & err & Winkler & Width & frac inf \\
                \hline
                CP & 0.099422 & 3.243817 & 2.400532 & 0.000782 \\
                NCCP & 0.099109 & 3.235151 & 2.400845 & 0.002189
            \end{tabular}
            \caption{Mean error (err), Winkler score, width and the fraction of intervals with infinite width for CP and NCCP for online processing of the Wine dataset. Note that the mean is computed for the finite intervals.}
            \label{tab:onlineWine}
        \end{table}

    \subsection{Offline classification}

        The results of the offline classification experiment on the USPS dataset are shown in Figure \ref{fig:offlineCalsizeUSPS}, where the mean OE and error over 1000 independent trials are shown together with a 95\% confidence interval. The ICP results are plotted as a function of the calibration set size, which ranges from 1\% to 99\%. As usual, the first thing to note, is that both methods achieve the finite-sample ACI guarantee \eqref{eq:ACIGuarantee} on average (see Figure \ref{fig:offlineUSPSErr}, as well as for each individual trial.

        As shown in Figure \ref{fig:offlineUSPSOE}, the ICP achieves a rather stable average OE of approximately 0.6 for most calibration set sizes, which means that most prediction sets contain excessive labels (i.e. labels other than the true one). On the other hand, INCCP achieves an average OE of approximately 0.05, meaning that there are very few incorrect labels in the prediction sets, and there is little variation over the trials, as the training set remains unchanged because no calibration set is required. Thus, the only source of randomness is the random seed in the \texttt{RandomForestClassifier}. For ICP, on the other hand, we use a different random split for each seed to determine the proper training set, which introduces more variability. The performance of the ICP deteriorates towards both ends of the plot in Figure \ref{fig:offlineUSPSOE}. For very small calibration set sizes, the resolution of the p-values becomes poor, and small changes in the significance level will have little effect. On the other hand, if almost all training data are used for calibration, the prediction rule (the trained \texttt{RandomForestClassifier} is likely to be very poor. In summary, ICCP outperforms ICP for all calibration set sizes in terms of OE for the USPS dataset.
        
        \begin{figure}[h!]
            \centering
            \begin{subfigure}[b]{0.48\linewidth}
                \includegraphics[width=\linewidth]{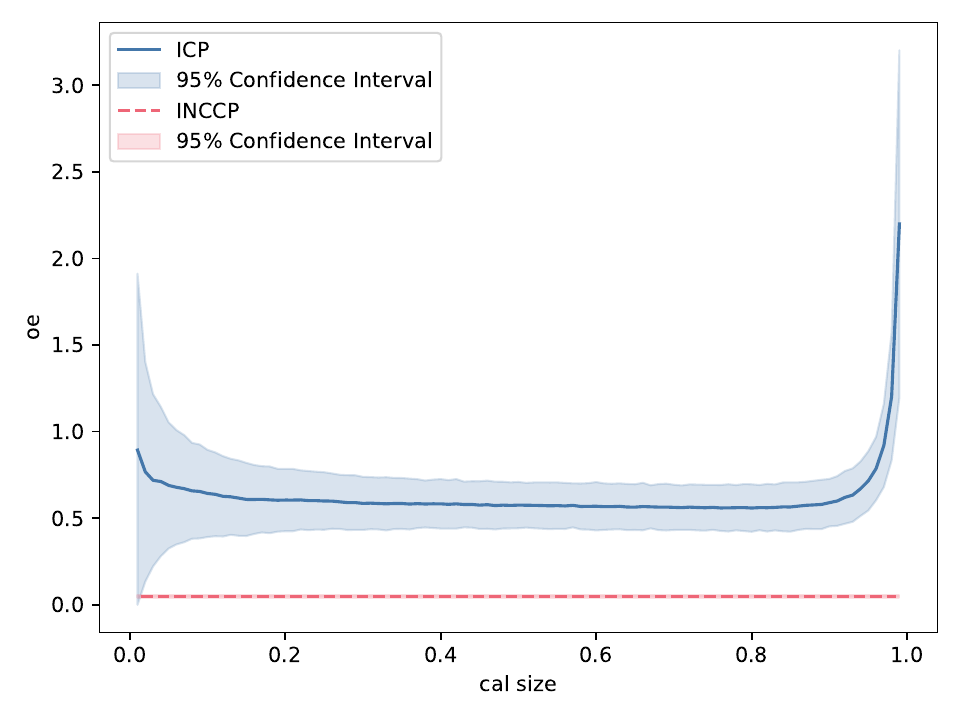}
                \caption{Observed excess (OE) over calibration set sizes for ICP and INCCP (which has no calibration set).}
                \label{fig:offlineUSPSOE}
            \end{subfigure}
            \hfill 
            \begin{subfigure}[b]{0.48\linewidth}
                \includegraphics[width=\linewidth]{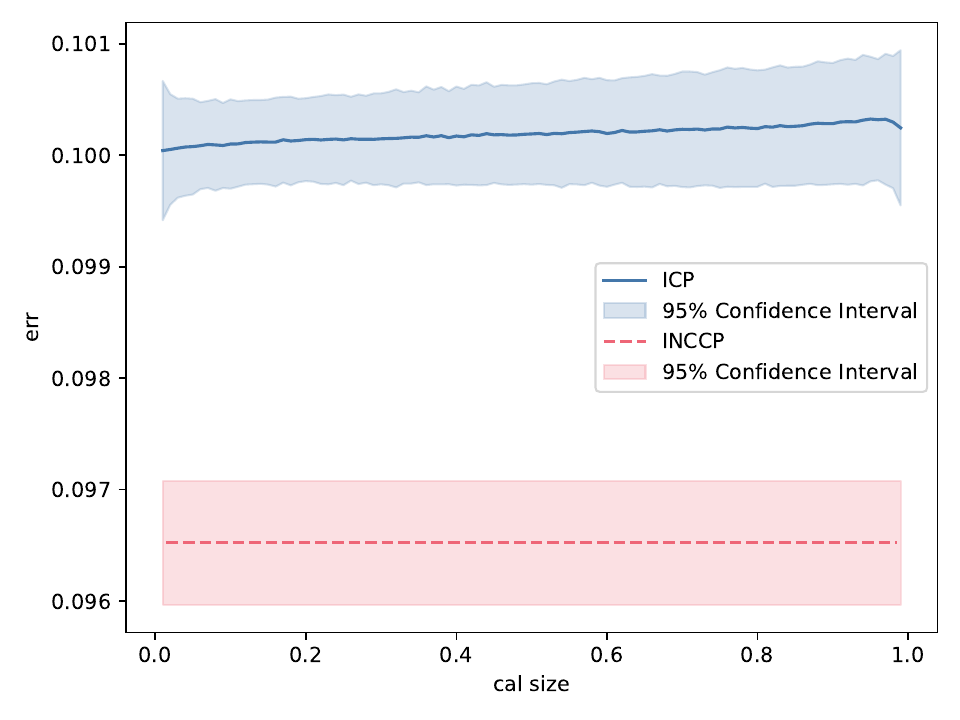}
                \caption{Average error rate over calibration set sizes for ICP and INCCP (which has no calibration set).}
                \label{fig:offlineUSPSErr}
            \end{subfigure}
            \caption{Average results over 1000 independent trials for ICP and INCCP over 99 calibration set sizes for ICP on the USPS dataset.}
            \label{fig:offlineCalsizeUSPS}
        \end{figure}
        
    \subsection{Offline regression}
    
        The results of the offline regression experiment on the Wine dataset are shown in Figure \ref{fig:offlineWine}, where the mean Winkler score, width, fraction of infinite intervals, and error rate are shown, together with 95\% confidence intervals. It should be noted that all mean values are computed only for the intervals that have finite widths, which complicates comparisons because no method was able to completely avoid infinite intervals, which means that the true mean values for the Winkler score and width are infinite in all cases. As usual, we begin by noting that both methods achieve the finite-sample ACI guarantee \eqref{eq:ACIGuarantee} on average (see Figure \ref{fig:offlineWineErr}), as well as for each individual trial.

        From Figures \ref{fig:offlineWineWinkler} and \ref{fig:offlineWineFracInf}, we can see that increasing the calibration set size seems to increase the ICP performance, as measured by the Winkler score and fraction of infinite intervals. The decrease in the mean Winkler score is almost monotonic, while inspecting the average width in Figure \ref{fig:offlineWineWidth} shows that the average width initially increases with increasing calibration size. The monotonic improvement in mean Winkler score of the ICP is somewhat surprising, since for very large calibration set sizes, the \texttt{RandomForestQuantileRegressor} is trained using very few examples, roughly 50 examples in the extreme case. It should be noted that the variation over 1000 independent trials increases with increasing calibration set size, as illustrated by the confidence intervals. On the other hand, the NCCP shows less variation and requires no calibration. The mean fraction of infinite intervals output by NCCP is consistently lower than that of ICP, as is the Winkler score, except for the extreme case where we use 99\% of the training set for calibration, where the mean Winkler score of the ICP is just slightly lower than that of NCCP.
                
        \begin{figure}[h!]
            \centering
            \begin{subfigure}[b]{0.48\linewidth}
                \includegraphics[width=\linewidth]{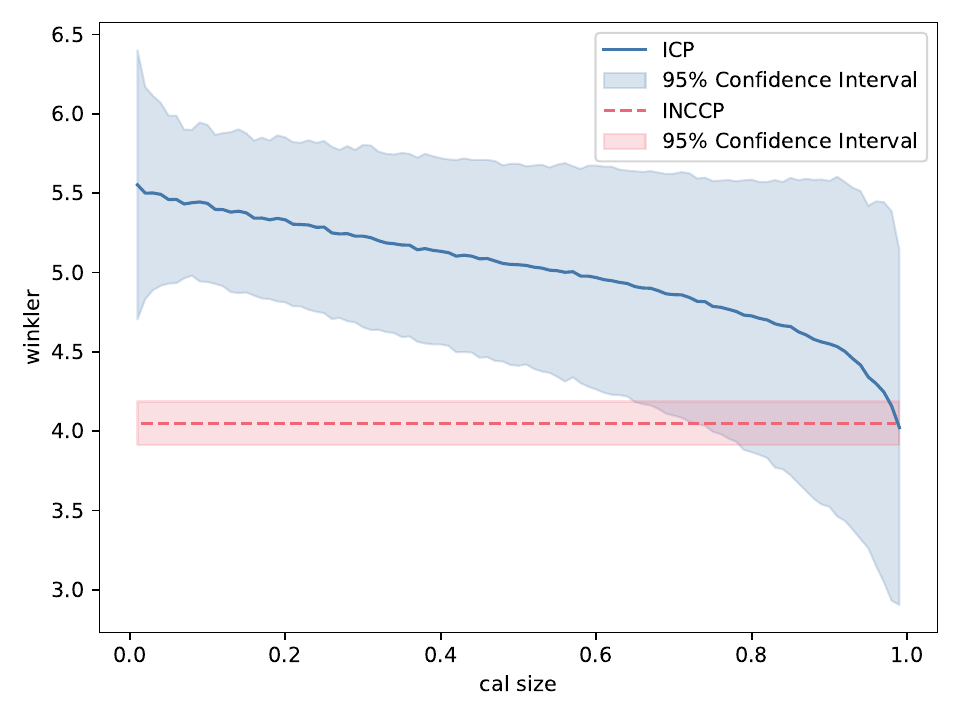}
                \caption{Mean Winkler score over calibration set sizes for ICP and INCCP (which has no calibration set).}
                \label{fig:offlineWineWinkler}
            \end{subfigure}
            \hfill 
            \begin{subfigure}[b]{0.48\linewidth}
                \includegraphics[width=\linewidth]{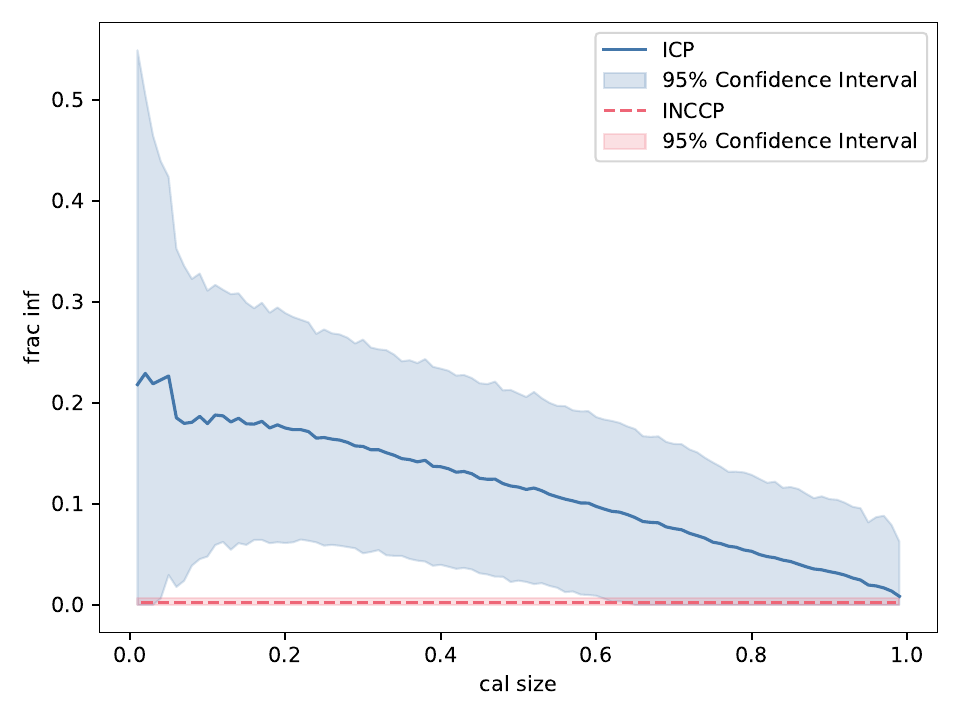}
                \caption{Average fraction of infinite intervals over calibration set sizes for ICP and INCCP (which has no calibration set).}
                \label{fig:offlineWineFracInf}
            \end{subfigure}
        
            \bigskip 
        
            \begin{subfigure}[b]{0.48\linewidth}
                \includegraphics[width=\linewidth]{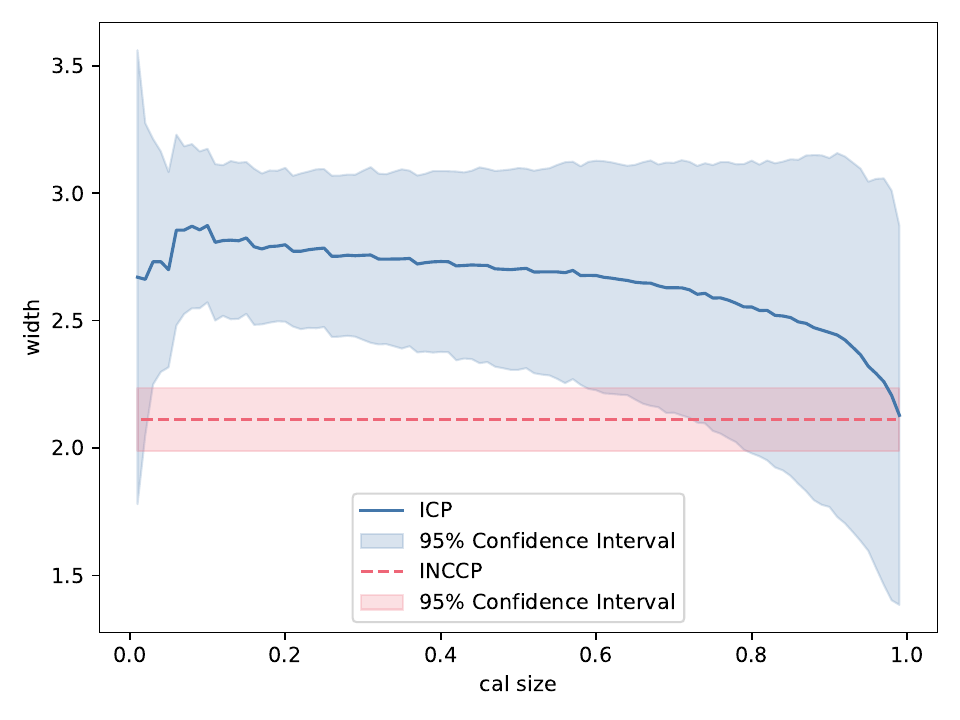}
                \caption{Mean width over calibration set sizes for ICP and INCCP (which has no calibration set).}
                \label{fig:offlineWineWidth}
            \end{subfigure}
            \hfill 
            \begin{subfigure}[b]{0.48\linewidth}
                \includegraphics[width=\linewidth]{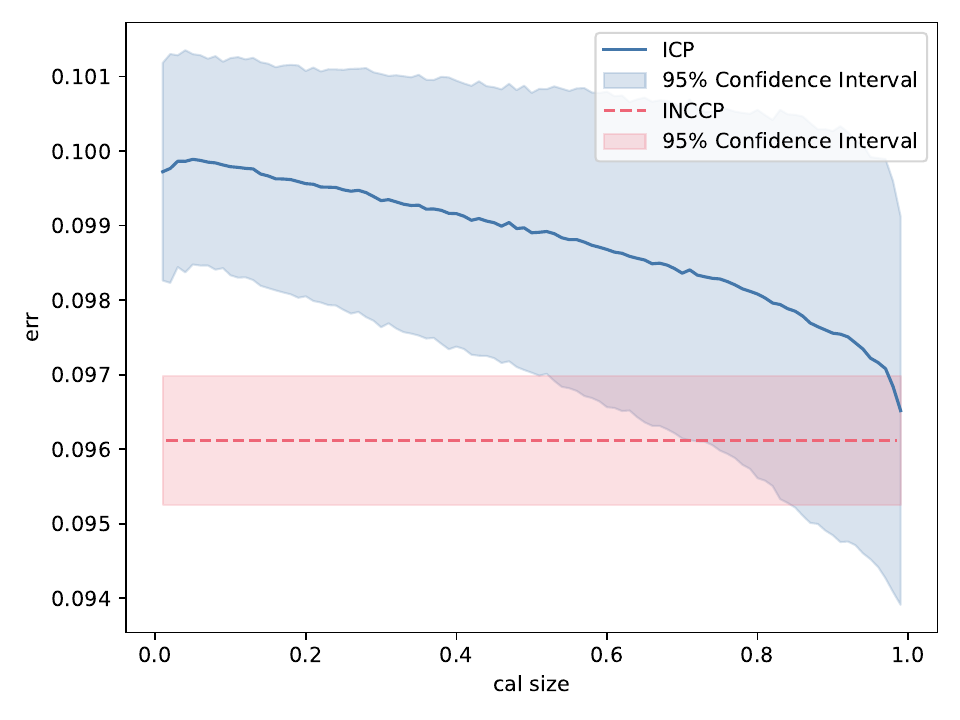}
                \caption{Mean error over calibration set sizes for ICP and INCCP (which has no calibration set).}
                \label{fig:offlineWineErr}
            \end{subfigure}
            \caption{Average results over 1000 independent trials for ICP and INCCP over 99 calibration set sizes for ICP on the Wine dataset. The plots show the mean error, Winkler score, width, and fraction of intervals with infinite width for ICP and INCCP for the Wine dataset. Note that the mean is computed for the finite intervals.}
            \label{fig:offlineWine}
        \end{figure}

\section{Discussion and Conclusions}\label{sec:conclusions}
    Adaptive conformal inference (ACI) was introduced to retain the asymptotic validity of conformal predictors under non-exchangeable data as well as some form of finite sample coverage guarantee \eqref{eq:ACIGuarantee}.
    In this paper, we have demonstrated that ACI does not rely on any specific properties of conformal predictors to achieve its finite sample guarantee. In fact, it does not even require the more general concept of a confidence predictor, where the prediction sets are required to be nested (see Definition \ref{def:confPred}). 
    However, we have argued that the property of nested prediction sets is the very least one should require when predicting with a confidence level. Without it, the coin-flip predictor (Definition \ref{def:coinFlipPredictor}) is exactly valid, but rather unhelpful. 

    Our findings highlight that the success of ACI in handling non-exchangeable data is primarily owing to its nature as an adaptive controller for the significance level, effectively steering the empirical error rate towards the target $\varepsilon$ over time. This control-theoretic view explains its robustness when applied to various types of confidence predictors, not just CP.
    
    We have mentioned several ways to construct non-conformal confidence predictors (NCCP) using popular machine learning methods, such as confidence thresholding, prediction intervals from parametric models, and quantile regression (with care to ensure nested prediction sets). 

    Because the validity guarantees of CP are lost if the exchangeability assumption is violated, and ACI provides finite sample coverage guarantees of another kind, we asked if anything is gained by using ACI with a CP over using it with an NCCP in situations where exchangeability cannot be expected. 

    \subsection{The online setting}
        In the online setting, compared to full conformal prediction (CP), NCCP is often much less computationally intense, and CP is often infeasible to implement, particularly in regression problems apart from some special cases (see, e.g. \cite{burnaev2014efficiency} and \cite{lei2019fast}). Our experiments on the USPS and the Wine datasets found little reason to prefer CP over NCCP in the online setting, as the latter performed as well or almost as well in terms of observed excess (OE) and Winkler score at a fraction of the computational cost. 

        Of course, if the exchangeability assumption holds, most NCCPs are not valid; however, if used together with ACI, the finite sample guarantee \eqref{eq:ACIGuarantee} still holds, as it makes no assumption about the data-generating distribution. Thus, if \eqref{eq:ACIGuarantee} is all that is desired, one may choose to use NCCP over CP, even in cases where the data are exchangeable. The principal reason for doing so would likely be to save computational resources. However, it is important to distinguish between the types of coverage guarantee. In the online setting, full CP guarantees that errors happen independently with a user-specified probability, which is much stronger than \eqref{eq:ACIGuarantee}.

    \subsection{The offline setting}
        In the offline setting, the main question was whether it is worth sacrificing some of the training data for calibration when the exchangeability assumption is violated. We do not gain any theoretical guarantee beyond what is provided by ACI to motivate this sacrifice, and one might suspect that it could pay to use all available data for training. Our experiments on the USPS and Wine datasets using inductive conformal predictors (ICP) and inductive non-conformal confidence predictors (INCCP) support this intuition. INCCP consistently outperformed ICP on all evaluation metrics for all calibration set sizes, with the exception of a minor improvement in the Winkler score for ICP on the Wine dataset when using 99\% of the training data for calibration. Moreover, the performance of INCCP was less variable over 1000 independent trials than that of ICP, which can likely be attributed to the lack of random splitting of the training data into proper training and calibration sets.

    \subsection{Future directions}
        In summary, further empirical studies on more datasets and predictors are needed. By far, the most widely used type of conformal predictors are ICPs owing to their more modest computational requirements. 
        The potential advantage of INCCP over ICP is that some training data must be set aside for calibration of the ICP, which can instead be used to train the INCCP. For large datasets, this may not be a major problem for ICP; however, if data are scarce, the calibration set may be better utilised to train an INCCP. Another domain where the need for a calibration set could be detrimental is time-series forecasting, in which the calibration set must be located after the proper training set. An INCCP then has the advantage of being trained on all data up to the point where the forecast begins, whereas the ICP has been trained on data that extends only to the start of the calibration set, in practice extending the forecast horizon by the size of the calibration set.
        Future studies should compare several ICPs based on popular machine learning methods and their INCCP counterparts. Both regression and classification datasets should be considered as well as time-series. 
        Another interesting subject for future empirical work is to compare INCCP with cross-conformal predictors \cite{vovk2015cross}, which is a hybrid of ICP and cross-validation that uses all available training data for both training and calibration.

        The finite-sample ACI guarantee \eqref{eq:ACIGuarantee} depends on the initial condition $\varepsilon_1$ of the ACI iteration \eqref{eq:ACI}. When choosing the step size according to \eqref{eq:gammaChoice}, one may be tempted to set $\varepsilon_1 = 1/2$ to minimise the required step size $\gamma$ that is required to achieve the chosen absolute error deviation bound $\delta$. However, this may not be the best choice for initial conditions from the perspective of efficient prediction sets. Future work should include methods for choosing a suitable initial condition for the ACI iteration. 

        Another direction of research is to improve ACI, as has been done in \cite{gibbs2024conformal} and \cite{pmlr-v235-podkopaev24a}. These studies focused on adaptive step size. Yet another possibility is to use the objects $x_n$ to either vary the step size or learn the optimal significance levels online.

        Finally, the idea of varying the significance level of the prediction sets to achieve certain goals is essentially a control problem. This problem falls under model-free control, which is concerned with controlling systems that have no explicit mathematical model. This was noted in \cite{angelopoulos2024conformal}, whose conformal PID method may be seen as a generalisation of ACI based on the classical PID control method. An effort to ensure the desired coverage while minimising the prediction set size may include incorporating ideas from model-free optimal control \cite{LAI2023110685}.

        \subsection{Conclusions}
        In conclusion, the main points of this paper can be summarised as follows:
        \begin{itemize}
            \item 
            \textbf{There is nothing conformal about ACI}\\
            The finite sample guarantee of ACI does not rely on the specific properties of CP, not even on NCCP, even if we argue that the latter is necessary in making it a versatile tool in non-exchangeable data settings.
            \item
            \textbf{Computational efficiency}\\
            NCCPs in the online setting are often significantly less computationally costly than CP, making them an attractive alternative for non-exchangeable data. We speculate that there may even be cases in which their (often) superior computational efficiency can make them preferable for exchangeable data.

            \item
            \textbf{Predictive efficiency}\\
            Our experimental results indicate that INCCP can outperform ICP on \\
            non-exchangeable data, as measured using popular evaluation metrics. However, more empirical results are required to determine the preferable method.

            \item 
            \textbf{Practical trade-offs}\\
            When choosing between CP and NCCP, different considerations are required in the online and batch settings. In both settings, depending on the specific NCCP, it may be more difficult to choose a suitable initial significance level. A natural starting point for CP is the target error rate $\eps$; however, this may be far from optimal for a general NCCP, depending on how the prediction sets are produced.
            \begin{itemize}
                \item 
                Online setting: The main consideration in the online setting is computational efficiency, where NCCP may be orders of magnitude faster. 
                \item 
                Batch setting: The principal disadvantage of ICP in the batch setting is that some part of the training set must be sacrificed for calibration. As mentioned in the discussion, if training data are scarce or in time-series forecasting, this could provide an edge to NCCP.
            \end{itemize}
            
        \end{itemize}

\section*{Acknowledgements}
    We are grateful to Vilma Ylv\'en and Emely Wiegand for their help in designing and producing the graphical abstract.
    The authors acknowledge the Swedish Knowledge Foundation and industrial partners for financially supporting the research and education environment on Knowledge Intensive Product Realisation SPARK at Jönköping University, Sweden. Project: PREMACOP grant no. 20220187.

\clearpage
\bibliographystyle{elsarticle-num} 
\bibliography{references}

\end{document}